\documentclass{article}

\usepackage{times}
\usepackage{graphicx} 
\usepackage{subfigure}

\usepackage{natbib}

\usepackage{algorithm}
\usepackage{algorithmic}
\usepackage{amssymb}
\usepackage{amsopn}
\usepackage{mathtools}
\DeclarePairedDelimiter\ceil{\lceil}{\rceil}

\usepackage{multicol}
\usepackage{amsthm}

\usepackage{hyperref}

\usepackage[accepted]{icml2017}

\newcommand{\wtf}[1]{\iffalse\textsf{\color{green}#1}\fi}

\newcommand{\ignore}[1]{}

\def\epsilon{\varepsilon}

\newcommand\av[2][]{\ensuremath{\mathbb{E}}_{#1}\left[ #2 \right]}

\newcommand{\innerprod}[2]{\left\langle #1, #2 \right\rangle}

\def\hessinv{\nabla^{-2}}
\def\hess{\nabla^2}

\newcommand{\defeq}{\triangleq}

\def\reals{{\mathbb R}}

\def\X{{\mathcal X}}
\def\Y{{\mathcal Y}}

\newtheorem{theorem}{Theorem}[section]

\newtheorem{lemma}[theorem]{Lemma}
\newtheorem{corollary}[theorem]{Corollary}

\newtheorem{fact}[theorem]{Fact}

\newtheorem{definition}[theorem]{Definition}

\newtheorem{remark}[theorem]{Remark}

\usepackage{multirow}
\renewcommand{\algorithmicrequire}{\textbf{Input:}}
\renewcommand{\algorithmicensure}{\textbf{Output:}}

\icmltitlerunning{The Price of Differential Privacy for Online Learning}

\begin{document}

\twocolumn[
\icmltitle{The Price of Differential Privacy for Online Learning}

\icmlsetsymbol{equal}{*}

\begin{icmlauthorlist}
\icmlauthor{Naman Agarwal}{pu}
\icmlauthor{Karan Singh}{pu}
\end{icmlauthorlist}

\icmlaffiliation{pu}{Computer Science, Princeton University, Princeton, NJ, USA}

\icmlcorrespondingauthor{Naman Agarwal}{namana@cs.princeton.edu}
\icmlcorrespondingauthor{Karan Singh}{karans@cs.princeton.edu}

\icmlkeywords{online learning, differential privacy, multi-armed bandits, optimal regret}

\vskip 0.3in
]

\printAffiliationsAndNotice{\icmlEqualContribution} 

\newpage
\begin{abstract}
We design differentially private algorithms for the problem of online linear optimization in the full information and bandit settings with optimal $\tilde{O}(\sqrt{T})$\footnote{Here the $\tilde{O}(\cdot)$ notation hides $polylog(T)$ factors.} regret bounds. In the full-information setting, our results demonstrate that $\epsilon$-differential privacy may be ensured for free -- in particular, the regret bounds scale as $O(\sqrt{T})+\tilde{O}\left(\frac{1}{\epsilon}\right)$. For bandit linear optimization, and as a special case, for non-stochastic multi-armed bandits, the proposed algorithm achieves a regret of $\tilde{O}\left(\frac{1}{\epsilon}\sqrt{T}\right)$, while the previously known best regret bound was $\tilde{O}\left(\frac{1}{\epsilon}T^{\frac{2}{3}}\right)$.
\end{abstract}


\section{Introduction}
In the paradigm of online learning, a learning algorithm makes a sequence of predictions given the (possibly incomplete) knowledge of the correct answers for the past queries. In contrast to statistical learning, online learning algorithms typically offer distribution-free guarantees. Consequently, online learning algorithms are well suited to dynamic and adversarial environments, where real-time learning from changing data is essential making them ubiquitous in practical applications such as servicing search advertisements. In these settings often these algorithms interact with sensitive user data, making privacy a natural concern for these algorithms. A natural notion of privacy in such settings is differential privacy \cite{dp} which ensures that the outputs of an algorithm are indistinguishable in the case when a user's data is present as opposed to when it is absent in a dataset.

In this paper, we design differentially private algorithms for online linear optimization with near-optimal regret, both in the full information and partial information (bandit) settings. This result improves the known best regret bounds for a number of important online learning problems -- including {\em prediction from expert advice} and {\em non-stochastic multi-armed bandits}.

\subsection{Full-Information Setting: Privacy for Free}
For the full-information setting where the algorithm gets to see the complete loss vector every round, we design $\epsilon$-differentially private algorithms with regret bounds that scale as $O\left(\sqrt{T}\right)+\tilde{O}\left(\frac{1}{\epsilon}\right)$ (Theorem \ref{thm:mainftrlthm}), partially resolving an open question to improve the previously best known bound of $O\left(\frac{1}{\epsilon}\sqrt{T}\right)$ posed in \cite{thakurnearly}. A decomposition of the bound on the regret bound of this form implies that when $\epsilon \geq \frac{1}{\sqrt{T}}$, the regret incurred by the differentially private algorithm matches the optimal regret in the non-private setting, i.e. differential privacy is {\em free}. Moreover even when $\epsilon \leq \frac{1}{\sqrt{T}}$, our results guarantee a sub-constant regret per round in contrast to the vacuous constant regret per round guaranteed by existing results.

Concretely, consider the case of online linear optimization over the cube, with unit $l_\infty$-norm-bounded loss vectors. In this setting, \cite{thakurnearly} achieves a regret bound of $O(\frac{1}{\epsilon}\sqrt{NT})$, which is meaningful only if $T\geq \frac{N}{\epsilon^2}$. Our theorems imply a regret bound of $\tilde{O}(\sqrt{NT} + \frac{N}{\epsilon})$. This is an improvement on the previous bound regardless of the value of $\epsilon$. Furthermore, when $T$ is between $\frac{N}{\epsilon}$ and $\frac{N}{\epsilon^2}$, the previous bounds are vacuous whereas our results are still meaningful. Note that the above arguments show an improvement over existing results even for moderate value of $\epsilon$. Indeed, when $\epsilon$ is very small, the magnitude of improvements are more pronounced. 

Beyond the separation between $T$ and $\epsilon$, the key point of our analysis is that we obtain bounds for general regularization based algorithms which adapt to the geometry of the underlying problem optimally, unlike the previous algorithms \cite{thakurnearly} which utilizes euclidean regularization. This allows our results to get rid of a polynomial dependence on $N$ (in the $\sqrt{T}$ term) in some cases. Online linear optimization over the sphere and prediction with expert advice are notable examples.

We summarize our results in Table \ref{t1}.

\begin{table*}[t]
\label{t1}
\vskip 0.15in
\begin{center}
\begin{small}
\begin{sc}
 \begin{tabular}{|| c | c | c | c ||}
 \hline
 Function Class ($N$ Dimensions)& \multirow{2}{*}{\parbox{3cm}{Previous Best Known Regret}} & Our Regret Bound & \multirow{2}{*}{\parbox{2cm}{Best Non-private Regret}} \\ [0.5ex]
 & & & \\
 & & & \\
 \hline\hline
 Prediction with Expert Advice & $\tilde{O}\left(\frac{\sqrt{T\log N}}{\epsilon}\right)$ & $O\left(\sqrt{T\log N} + \frac{N\log N \log^2 T}{\epsilon} \right)$ & $O(\sqrt{T\log N})$\\
 & & & \\
 \hline
 \multirow{2}{*}{\parbox{5cm}{Online linear optimization over the Sphere}} & $\tilde{O}\left(\frac{\sqrt{NT}}{\epsilon}\right)$ & $O\left(\sqrt{T} + \frac{N \log^2 T}{\epsilon} \right)$ & $O(\sqrt{T})$\\
 & & & \\
 \hline
  \multirow{2}{*}{\parbox{5cm}{Online linear optimization over the Cube}} & $\tilde{O}\left(\frac{\sqrt{NT}}{\epsilon}\right)$ & $O\left(\sqrt{NT} + \frac{N \log^2 T}{\epsilon} \right)$ & $O(\sqrt{NT})$\\
 & & & \\
 \hline
 Online Linear Optimization & $\tilde{O}\left(\frac{\sqrt{T}}{\epsilon}\right)$ & $O\left( \sqrt{T} + \frac{\log^2 T}{\epsilon} \right)$ & $O(\sqrt{T})$\\ [1ex]
 \hline
\end{tabular}
\end{sc}
\end{small}
\end{center}
\caption{Summary of our results in the full-information setting. In the last row we suppress the constants depending upon $\X, \Y$.}
\vskip -0.1in
\end{table*}

\subsection{Bandits: Reduction to the Non-private Setting}
In the partial-information (bandit) setting, the online learning algorithm only gets to observe the loss of the prediction it prescribed. We outline a reduction technique that translates a non-private bandit algorithm to a differentially private bandit algorithm, while retaining the $\tilde{O}(\sqrt{T})$ dependency of the regret bound on the number of rounds of play (Theorem~\ref{thm:LinearBanditGeneral}). This allows us to derive the first $\epsilon$-differentially private algorithm for bandit linear optimization achieving $\tilde{O}(\sqrt{T})$ regret, using the algorithm for the non-private setting from \cite{AHR}. This answers a question from \cite{thakurnearly} asking if $\tilde{O}(\sqrt{T})$ regret is attainable for differentially private linear bandits .

An important case of the general bandit linear optimization framework is the {\em non-stochastic multi-armed bandits} problem\cite{bubook}, with applications for website optimization, personalized medicine, advertisement placement and recommendation systems. Here, we propose an $\epsilon$-differentially private algorithm which enjoys a regret of $\tilde{O}(\frac{1}{\epsilon}\sqrt{NT\log N})$ (Theorem \ref{thm:mab}), improving on the previously best attainable regret of $\tilde{O}(\frac{1}{\epsilon} NT^{\frac{2}{3}})$\cite{thakurnearly}.

We summarize our results in Table \ref{t2}.

\begin{table*}
\label{t2}
\vskip 0.15in
\begin{center}
\begin{small}
\begin{sc}
 \begin{tabular}{|| c | c | c | c ||}
 \hline
 Function Class ($N$ Dimensions)& \multirow{2}{*}{\parbox{3cm}{Previous Best Known Regret}} & Our Regret Bound & \multirow{2}{*}{\parbox{2cm}{Best Non-private Regret}} \\ [0.5ex]
 & & & \\
 & & & \\
 \hline\hline
 Bandit Linear Optimization  & $\tilde{O}\left(\frac{T^{\frac{2}{3}}}{\epsilon}\right)$ & $\tilde{O}\left(\frac{\sqrt{T}}{\epsilon}\right)$ & $O(\sqrt{T})$\\
 & & & \\
 \hline
 \multirow{2}{*}{\parbox{5cm}{Non-stochastic Mult-armed Bandits}} & $\tilde{O}\left(\frac{NT^{\frac{2}{3}}}{\epsilon}\right)$ & $\tilde{O}\left(\frac{\sqrt{TN\log N}}{\epsilon}\right)$ & $O(\sqrt{NT})$\\
 & & & \\
 \hline
\end{tabular}
\end{sc}
\end{small}
\end{center}
\vskip -0.1in
\caption{Summary of our results in the bandit setting. In the first row we suppress the specific constants depending upon $\X, \Y$.}
\end{table*}

\subsection{Related Work}
The problem of differentially private online learning was first considered in \cite{con}, albeit guaranteeing privacy in a weaker setting -- ensuring the privacy of the individual entries of the loss vectors. \cite{con} also introduced the tree-based aggregation scheme for releasing the cumulative sums of vectors in a differentially private manner, while ensuring that the total amount of noise added for each cumulative sum is only poly-logarithmically dependent on the number of vectors. The stronger notion of privacy protecting entire loss vectors was first studied in \cite{jaink}, where gradient-based algorithms were proposed that achieve $(\epsilon,\delta)$-differntial privacy and regret bounds of $\tilde{O}\left(\frac{1}{\epsilon}\sqrt{T}\log \frac{1}{\delta}\right)$.  \cite{thakurnearly} proposed a modification of Follow-the-Approximate-Leader template to achieve $\tilde{O}\left(\frac{1}{\epsilon}\log^{2.5} T\right)$ regret for strongly convex loss functions, implying a regret bound of $\tilde{O}\left(\frac{1}{\epsilon}\sqrt{T}\right)$ for general convex functions. In addition, they also demonstrated that under bandit feedback, it is possible to obtain regret bounds that scale as $\tilde{O}\left(\frac{1}{\epsilon}T^{\frac{2}{3}}\right)$. \cite{dpbook, ja} proved that in the special case of {\em prediction with expert advice} setting, it is possible to achieve a regret of $O\left(\frac{1}{\epsilon}\sqrt{T\log N}\right)$. While most algorithms for differentially private online learning are based on the regularization template, \cite{tal} used a perturbation-based algorithm to guarantee $(\epsilon,\delta)$-differential privacy for the problem of online PCA. \cite{toss} showed that it is possible to design $\epsilon$-differentially private algorithms for the stochastic multi-armed bandit problem with a separation of $\epsilon, T$ for the regret bound. Recently, an independent work due to \cite{chu}, which we were made aware of after the first manuscript, also demonstrated a $\tilde{O}\left(\frac{1}{\epsilon}\sqrt{T}\right)$ regret bound in the \textit{non-stochastic multi-armed bandits} setting. We match their results (Theorem \ref{thm:mab}), as well as provide a generalization to arbitrary convex sets (Theorem \ref{thm:LinearBanditGeneral}). 
\subsection{Overview of Our Techniques}

\textbf{Full Information Setting:} We consider the two well known paradigms for online learning, \textit{Folllow-the-Regularized-Leader (FTRL)} and \textit{Folllow-the-Perturbed-Leader (FTPL)}. In both cases, we ensure differential privacy by restricting the mode of access to the inputs (the loss vectors). In particular, the algorithm can only retrieve estimates of the loss vectors released by a tree based aggregation protocol (Algorithm \ref{alg2}) which is a slight modification of the protocol used in \cite{jaink,thakurnearly}. We outline a tighter analysis of the regret minimization framework by crucially observing that in case of linear losses, the expected regret of an algorithm that injects identically (though not necessarily independently) distributed noise per step is the same as one that injects a single copy of the noise at the very start of the algorithm. 

The regret analysis of Follow-the-Leader based algorithm involves two components, a {\em bias} term due to the regularization and a {\em stability} term which bounds the change in the output of the algorithm per step. In the analysis due to \cite{thakurnearly}, the stability term is affected by the variance of the noise as it changes from step to step. However in our analysis, since we treat the noise to have been sampled just once, the stability analysis does not factor in the variance and the magnitude of the noise essentially appears as an additive term in the bias.

\textbf{Bandit Feedback:} In the bandit feedback setting, we show a general reduction that takes a non-private algorithm and outputs a private algorithm (Algorithm \ref{alg:DPBanditGeneral}). Our key observation here (presented as Lemma \ref{lemma:noisyoco}) is that on linear functions, in expectation the regret of an algorithm on a noisy sequence of loss vectors is the same as its regret on the original loss sequence as long as noise is zero mean. We now bound the regret on the noisy sequence by conditioning out the case when the noise can be large and using exploration techniques from \cite{bubecklinearbandit} and \cite{dark}. 

\section{Model and Preliminaries}
This section introduces the model of online (linear) learning, the distinction between full and partial feedback scenarios, and the notion of differential privacy in this model.

{\bf Full-Information Setting:} {\em Online linear optimization} \cite{elad, oobook} involves repeated decision making  over $T$ rounds of play. At the beginning of every round (say round $t$), the algorithm chooses a point in $x_t\in \mathcal{X}$, where $\mathcal{X}\subseteq \mathbb{R}^N$ is a (compact) convex set. Subsequently, it observes the loss $l_t\in \mathcal{Y}\subseteq\mathbb{R}^N$ and suffers a loss of $\langle l_t, x_t\rangle$. The success of such an algorithm, across $T$ rounds of play, is measured though {\bf regret}, which is defined as
\[
\textrm{Regret} = \mathbb{E}\bigg[\sum_{t=1}^T \langle l_t, x_t\rangle - \min_{x\in \mathcal{K}} \sum_{t=1}^T \langle l_t, x\rangle\bigg]
\]
where the expectation is over the randomness of the algorithm. In particular, achieving a sub-linear regret ($o(T)$) corresponds to doing almost as good (averaging across $T$ rounds) as the fixed decision with the least loss in hindsight. In the non-private setting, a number of algorithms have been devised to achieve $O(\sqrt{T})$ regret, with additional dependencies on other parameters dependent on the properties of the specific decision set $\mathcal{X}$ and loss set $\mathcal{Y}$. (See \cite{elad} for a survey of results.)

Following are three important instantiations of the above framework. 
\begin{itemize}
	\item {\em Prediction with Expert Advice}: Here the underlying decision set is the simplex $\mathcal{X}=\Delta_N = \{x\in\mathbb{R}^n: x_i\geq 0, \sum_{i=1}^n x_i = 1\}$ and the loss vectors are constrained to the unit cube $\mathcal{Y}=\{l_t\in\mathbb{R}^N: \|l_t\|_\infty\leq 1\}$.
	\item {\em OLO over the Sphere}: Here the underlying decision is the euclidean ball $\mathcal{X} = \{x\in\mathbb{R}^n: \|x\|_2 \leq 1\}$ and the loss vectors are constrained to the unit euclidean ball $\mathcal{Y}=\{l_t\in\mathbb{R}^N: \|l_t\|_2 \leq 1\}$.
	\item {\em OLO over the Cube}: The decision is the unit cube $\mathcal{X} = \{x\in\mathbb{R}^n: \|x\|_\infty \leq 1\}$, while the loss vectors are constrained to the set $\mathcal{Y}=\{l_t\in\mathbb{R}^N: \|l_t\|_1 \leq 1\}$.
\end{itemize}

{\bf Partial-Information Setting:} In the setting of bandit feedback, the critical difference is that the algorithm only gets to observe the value $\langle l_t, x_t\rangle$, in contrast to the complete loss vector $l_t\in\mathbb{R}^N$ as in the full information scenario. Therefore, the only feedback the algorithm receives is the value of the loss it incurs for the decision it takes. This makes designing algorithms for this feedback model challenging. Nevertheless for the general problem of bandit linear optimization, \cite{dark} introduced a computationally efficient algorithm that achieves an optimal dependence of the incurred regret of $O(\sqrt{T})$ on the number of rounds of play. The {\em non-stochastic multi-armed bandit} \cite{auer2002nonstochastic} problem is the bandit version of the {\em prediction with expert advice} framework.

{\bf Differential Privacy:} Differential Privacy \cite{dp} is a rigorous framework for establishing guarantees on privacy loss, that admits a number of desirable properties such as graceful degradation of guarantees under composition and robustness to linkage acts \cite{dpbook}. 

\begin{definition}[$(\epsilon, \delta)$-Differential Privacy]
A randomized online learning algorithm $\mathcal{A}$ on the action set $\mathcal{X}$ and the loss set $\mathcal{Y}$ is $(\epsilon,\delta)$-differentially private if for any two sequence of loss vectors $L=(l_1,\dots l_T)\subseteq \mathcal{Y}^T$ and $L'=(l'_1,\dots l'_T)\subseteq\mathcal{Y}^T$ differing in at most one vector -- that is to say $\exists t_0\in [T], \forall t\in [T]-\{t_0\}, l_t=l'_t$ -- for all $S\subseteq \mathcal{X}^T$, it holds that
\begin{equation*}
\mathbb{P}(\mathcal{A}(L)\in S) \leq e^\epsilon \mathbb{P}(\mathcal{A}(L')\in S) + \delta
\end{equation*}
\end{definition}

\begin{remark}
	The above definition of Differential Privacy is specific to the online learning scenario in the sense that it assumes the change of a complete loss vector. This has been the standard notion considered earlier in \cite{jaink,thakurnearly}. Note that the definition entails that the entire sequence of predictions produced by the algorithm is differentially private.
\end{remark}

{\bf Notation:} We define $\|\mathcal{Y}\|_p = \max \{\|l_t\|_p: l_t\in \mathcal{Y}\}$, $\|\mathcal{X}\|_p = \max \{\|x\|_p: x\in \mathcal{X}\}$, and $M=\max_{l\in\mathcal{Y}, x\in\mathcal{X}}|\langle l,x\rangle|$, where $\|\cdot\|_p$ is the $l_p$ norm. By Holder's inequality, it is easy to see that $M\leq \|\mathcal{Y}\|_p\|\mathcal{X}\|_q$ for all $p,q\geq 1$ with $\frac{1}{p}+\frac{1}{q}=1$. We define the distribution $Lap^N(\lambda)$ to be the distribution over $\reals^N$ such that each coordinate is drawn independently from the Laplace distribution with parameter $\lambda$.


\section{Full-Information Setting: Privacy for Free}
In this section, we describe an algorithmic template (Algorithm \ref{alg:ftrlgeneric}) for differentially private online linear optimization, based on {\em Follow-the-Regularized-Leader} scheme. Subsequently, we outline the noise injection scheme (Algorithm \ref{alg2}), based on the Tree-based Aggregation Protocol \cite{con}, used as a subroutine by Algorithm \ref{alg:ftrlgeneric} to ensure input differential privacy. The following is our main theorem in this setting.

\begin{theorem}
\label{thm:mainftrlthm}
	Algorithm \ref{alg:ftrlgeneric} when run with $\mathcal{D} = Lap^N(\lambda)$ where $\lambda = \frac{\|\mathcal{Y}\|_1\log T}{\epsilon}$, regularization $R(x)$, decision set $\mathcal{X}$ and loss vectors $l_1, \ldots l_t$, the regret of Algorithm \ref{alg:ftrlgeneric} is bounded by 
	\[\textrm{Regret} \leq \sqrt{D_R \sum_{t = 1}^{T} \max_{x \in \mathcal{X}}(\|l_t\|_{\hess R(x)}^*)^2} + D_{Lap}\]
	where 
	\[D_{Lap} = \av[Z \sim \mathcal{D}']{\max_{x \in X} \innerprod{Z}{x} - \min_{x \in X} \innerprod{Z}{x}}\]
	\[D_{R} = \max_{x \in X} R(x) - \min_{x \in X} R(x)\]
	and $\mathcal{D}'$ is the distribution induced by the sum of $\ceil{\log T}$ independent samples from $\mathcal{D}$, $\|\cdot\|^*_{\hess R(x)}$ represents the dual of the norm with respect to the hessian of $R$.  
	Moreover, the algorithm is $\epsilon$-differentially private, i.e. the sequence of predictions produced $(x_t:t\in [T])$ is $\epsilon$-differentially private.
\end{theorem}

\begin{algorithm}
\caption{FTRL Template for OLO}
\label{alg:ftrlgeneric}
\begin{algorithmic}[1]
	\INPUT Noise distribution $\mathcal{D}$, Regularization $R(x)$
	\STATE Initialize an empty binary tree $B$ to compute differentially private estimates of $\sum_{s=1}^t l_s$.
	\STATE Sample $n_0^1, \dots n_0^{\ceil{\log T}}$ independently from $\mathcal{D}$.
	\STATE $\tilde{L}_0 \leftarrow \sum_{i=1}^{\ceil{\log T}} n_0^{i}$.
	\FOR{$t=1$ to $T$}
	\STATE Choose $x_t = argmin_{x\in\mathcal{X}} \left(\eta\langle x, \tilde{L}_{t-1}\rangle+R(x)\right)$.
		\STATE Observe $l_t\in \mathcal{Y}$, and suffer a loss of $\langle l_t, x_t\rangle$.
		\STATE $(\tilde{L}_t, B) \leftarrow \texttt{TreeBasedAgg}(l_t, B, t, \mathcal{D}, T)$.
	\ENDFOR
\end{algorithmic}
\end{algorithm}

The above theorem leads to following corollary where we show the bounds obtained in specific instantiations of online linear optimization.

\begin{corollary}
Substituting the choices of $\lambda, R(x)$ listed below, we specify the regret bounds in each case.
\begin{enumerate}
		\item  \textbf{Prediction with Expert Advice: } Choosing $\lambda = \frac{N\log T}{\epsilon}$ and $R(x) = \sum_{i=1}^N x_i \log(x_i)$, 
\begin{equation*}
Regret \leq O\Bigg(\sqrt{T\log N} + \frac{N \log^2 T \log N}{\epsilon}\Bigg)
\end{equation*}
	\item \textbf{OLO over the Sphere} Choosing $\lambda = \frac{\sqrt{N}\log T}{\epsilon}$ and $R(x) = \|x\|_2^2$
\begin{equation*}
Regret \leq O\Bigg(\sqrt{T} + \frac{N \log^2 T}{\epsilon} \Bigg)
\end{equation*}
	\item \textbf{OLO over the Cube} With $\lambda = \frac{\log T}{\epsilon}$ and $R(x) = \|x\|_2^2$
\begin{equation*}
Regret \leq O\Bigg(\sqrt{NT} + \frac{N \log^2 T}{\epsilon} \Bigg)
\end{equation*}
	\end{enumerate}
\end{corollary}

\begin{algorithm}
\caption{$\texttt{TreeBasedAgg}(l_t, B, t, \mathcal{D}, T)$}
\label{alg2}
\begin{algorithmic}[1]
	\INPUT Loss vector $l_t$, Binary tree $B$, Round $t$, Noise distribution $\mathcal{D}$, Time horizon $T$
	\STATE $(\tilde{L'}_t, B) \leftarrow \texttt{PrivateSum}(l_t, B, t, \mathcal{D}, T)$ -- Algorithm 5 (\cite{jaink}) with the noise added at each node -- be it internal or leaf -- sampled independently from the distribution $\mathcal{D}$.
	\STATE $s_t\leftarrow$ the binary representation of $t$ as a string.
	\STATE Find the minimum set $\mathcal{S}$ of {\em already} populated nodes in $B$ that can compute $\sum_{s=1}^t l_s$. 
	\STATE Define $Q=|\mathcal{S}| \leq \ceil{\log T}$. Define $r_t=\ceil{\log T}-Q$.
	\STATE Sample $n_t^1, \dots n_t^{r_t}$ independently from $\mathcal{D}$.
	\STATE $\tilde{L}_t \leftarrow \tilde{L'}_t + \sum_{i=1}^{r_t} n_t^{i}$.
	\OUTPUT $(\tilde{L_t}, B)$.
\end{algorithmic}
\end{algorithm}

\subsection{Proof of Theorem \ref{thm:mainftrlthm}}

We first prove the privacy guarantee, and then prove the claimed bound on the regret. For the analysis, we define the random variable $Z_t$ to be the net amount of noise injected by the TreeBasedAggregation (Algorithm \ref{alg2}) on the true partial sums. Formally, $Z_t$ is the difference between cumulative sum of loss vectors and its differentially private estimate used as input to the arg-min oracle.
\begin{align*}
Z_t = \tilde{L}_t - \sum_{i=1}^t l_i
\end{align*}
Further, let $\mathcal{D}'$ be the distribution induced by summing of $\ceil{\log T}$ independent samples from $\mathcal{D}$.

\textbf{Privacy} : To make formal claims about the quality of privacy, we ensure {\em input differential privacy} for the algorithm -- that is, we ensure that the {\bf entire sequence} of partial sums of the loss vectors $(\sum_{s=1}^t l_s: t\in [T])$ is $\epsilon$-differentially private. Since the outputs of Algorithm \ref{alg:ftrlgeneric} are strictly determined by the prefix sum estimates produced by $\texttt{TreeBasedAgg}$, by the post-processing theorem, this certifies that the entire sequence of choices made by the algorithm (across all $T$ rounds of play) $(x_t: t\in [T])$ is $\epsilon$-differentially private. We modify the standard Tree-based Aggregation protocol to make sure that the noise on each output (partial sum) is distributed identically (though not necessarily independently) across time. While this modification is not essential for ensuring privacy, it simplifies the regret analysis.

\begin{lemma}[Privacy Guarantees with Laplacian Noise]
\label{thm:p1}
	Choose any $\lambda \geq \frac{\|\mathcal{Y}\|_1\log T}{\epsilon}$. When Algorithm \ref{alg2} $\mathcal{A}(\mathcal{D},T)$ is run with $\mathcal{D}=Lap^N(\lambda)$, the following claims hold true:
	\begin{itemize}
	\item {\bf Privacy:} The sequence $(\tilde{L}_t: t\in [T])$ is $\epsilon$-differentially private.
	\item {\bf Distribution:} $\forall t\in [T], Z_t \sim \sum_{i=1}^{\ceil{\log T}} n_i$, where each $n_i$ is independently sampled from $Lap^N(\lambda)$.
\end{itemize}\end{lemma}
 \begin{proof}
 By Theorem 9 (\cite{jaink}), we have that the sequence $(\tilde{L'}_t: t\in [T])$ is $\epsilon$-differentially private. Now the sequence $(\tilde{L}_t: t\in [T])$ is $\epsilon$-differentially private because differential privacy is immune to post-processing\cite{dpbook}. 

Note that the $\texttt{PrivateSum}$ algorithm adds exactly $|\mathcal{S}|$ independent draws from the distribution $\mathcal{D}$ to $\sum_{s=1}^t l_s$, where $\mathcal{S}$ is the minimum set of already populated nodes in the tree that can compute the required prefix sum. Due to Line 6 in Algorithm \ref{alg2}, it is made certain that every prefix sum released is a sum of the true prefix sum and $\ceil{\log T}$ independent draws from $\mathcal{D}$.
 \end{proof}

%


\textbf{Regret Analysis:}
In this section, we show that for linear loss functions any instantiation of the {\em Follow-the-Regularized-Leader} algorithm can be made differentially private with an additive loss in regret. 

\begin{theorem}
\label{thm:mainftrlthm2}
	For any noise distribution $\mathcal{D}$, regularization $R(x)$, decision set $\mathcal{X}$ and loss vectors $l_1, \ldots l_t$, the regret of Algorithm \ref{alg:ftrlgeneric} is bounded by 
	\[\textrm{Regret} \leq \sqrt{D_R \sum_{t = 1}^{T} \max_{x \in \mathcal{X}}(\|l_t\|_{\hess R(x)}^*)^2} + D_{\mathcal{D}'}\]
	where $D_{\mathcal{D}'} = \av[Z \sim \mathcal{D}']{\max_{x \in \X} \innerprod{Z}{x} - \min_{x \in \X} \innerprod{Z}{x}}$, $D_{R} = \max_{x \in \X} R(x) - \min_{x \in \X} R(x)$, and $\|\cdot\|^*_{\hess R(x)}$ represents the dual of the norm with respect to the hessian of $R$.  
\end{theorem}

\begin{proof}
To analyze the regret suffered by Algorithm \ref{alg:ftrlgeneric}, we consider an alternative algorithm that performs a one-shot noise injection -- this alternate algorithm may not be differentially private. The observation here is that the alternate algorithm and Algorithm \ref{alg:ftrlgeneric} suffer the same loss in expectation and therefore the same expected regret which we bound in the analysis below.

	Consider the following alternate algorithm which instead of sampling noise $Z_t$ at each step instead samples noise at the beginning of the algorithm and plays with respect to that. Formally consider the sequence of iterates $\hat{x}_t$ defined as follows. Let $Z \sim D$.
	\[ \hat{x}_1 \defeq x_1, \quad\hat{x}_t \defeq argmin_{x\in\mathcal{X}} \eta \langle x, Z+\sum_{i} l_i\rangle+ R(x)\]
	We have that 
	\begin{equation} 
	\label{eqn:equivalence}
	\av[Z_1 \ldots Z_T \sim D]{\sum_{t = 1}^T \innerprod{l_t}{x_t}} = \av[Z \sim D]{\sum_{t = 1}^T \innerprod{l_t}{\hat{x}_t}}
	\end{equation}
	To see the above equation note that $\av[Z_t \sim D]{\innerprod{l_t}{\hat{x}_t}} = \av[Z \sim D]{\innerprod{l_t}{x_t}}$ since $x, \hat{x}_t$ have the same distribution.

	Therefore it is sufficient to bound the regret of the sequence $\hat{x}_1 \ldots \hat{x}_t$. The key idea now is to notice that the addition of one shot noise does not affect the stability term of the FTRL analysis and therefore the effect of the noise need not be paid at every time step. Our proof will follow the standard template of using the FTL-BTL \cite{kv} lemma and then bounding the stability term in the standard way. 
	Formally define the augmented series of loss functions by defining 
	\[l_0(x) = \frac{1}{\eta} R(x) + \innerprod{Z}{x}\] where $Z \sim D$ is a sample.
	Now invoking the Follow the Leader, Be the Leader Lemma (Lemma 5.3, \cite{elad}) we get that for any fixed $u \in \mathcal{X}$
	\[ \sum_{t = 0}^{T} l_t(u) \geq \sum_{t = 0}^{T} l_t(\hat{x}_{t+1}) \]
	Therefore we can conclude that 
	\begin{align}
		&\sum_{t = 1}^{T}[l_t(\hat{x}_t) - l_t(u)] \\
		\leq& \sum_{t = 1}^T [l_t(\hat{x}_t) - l_t(\hat{x}_{t+1})] + l_0(u) - l_0(\hat{x}_1) \nonumber\\
		\leq& \sum_{t = 1}^T [l_t(\hat{x}_t) - l_t(\hat{x}_{t+1})] + \frac{1}{\eta} D_R + D_Z \label{eqn:pseudoregret}
	\end{align}
	where $D_Z \defeq \max_{x \in X}( \innerprod{Z}{x}) - \min_{x \in X}( \innerprod{Z}{x})$
	Therefore we now need to bound the stability term $l_t(\hat{x}_t) - l_t(\hat{x}_{t+1})$. Now, the regret bound follows from the standard analysis for the stability term in the FTRL scheme (see for instance \cite{elad}). Notice that the bound only depends on the change in the cumulative loss per step i.e. $\eta \left( \sum_t l_t + Z \right)$, for which the change is the loss vector $\eta l_{t+1}$ across time steps. Therefore we get that 
	\begin{equation}
	\label{eqn:stabilitybound}
	l_t(\hat{x}_t) - l_t(\hat{x}_{t+1}) \leq \max_{x \in X}\|\eta l_t\|_{\eta \hessinv R(x)}^2
	\end{equation}

	Combining Equations \eqref{eqn:equivalence}, \eqref{eqn:pseudoregret}, \eqref{eqn:stabilitybound} we get the regret bound in Theorem \ref{thm:mainftrlthm2}.

\end{proof}
\subsection{Regret Bounds for FTPL}
In this section, we outline algorithms based on the {\em Follow-the-Perturbed-Leader} template\cite{kv}. FTPL-based algorithms ensure low-regret by perturbing the cumulative sum of loss vectors with noise from a suitably chosen distribution. We show that the noise added in the process of FTPL is sufficient to ensure differential privacy. More concretely, using the regret guarantees  due to \cite{olo}, for the full-information setting, we establish that the regret guarantees obtained scale as $O(\sqrt{T})+\tilde{O}(\frac{1}{\epsilon}\log\frac{1}{\delta})$. While Theorem \ref{thm:ftplolo} is valid for all instances of online linear optimization and achieves $O(\sqrt{T})$ regret, it yields sub-optimal dependence on the dimension of the problem. The advantage of FTPL-based approaches over FTRL is that FTPL performs linear optimization over the decision set every round, which is possibly computationally less expensive than solving a convex program every round, as FTRL requires.

\begin{algorithm}
\caption{FTPL Template for OLO -- $\mathcal{A}(\mathcal{D}, T)$ on the action set $\mathcal{X}$, the loss set $\mathcal{Y}$.}
\label{alg1}
\begin{algorithmic}[1]
	\STATE Initialize an empty binary tree $B$ to compute differentially private estimates of $\sum_{s=1}^t l_s$.
	\STATE Sample $n_0^1, \dots n_0^{\ceil{\log T}}$ independently from $\mathcal{D}$.
	\STATE $\tilde{L}_0 \leftarrow \sum_{i=1}^{\ceil{\log T}} n_0^{i}$.
	\FOR{$t=1$ to $T$}
		\STATE Choose $x_t = argmin_{x\in\mathcal{X}} \langle x, \tilde{L}_{t-1}\rangle$.
		\STATE Observe the loss vector $l_t\in \mathcal{Y}$, and suffer $\langle l_t, x_t\rangle$.
		\STATE $(\tilde{L}_t, B) \leftarrow \texttt{TreeBasedAgg}(l_t, B, t, \mathcal{D}, T)$.
	\ENDFOR
\end{algorithmic}
\end{algorithm}

\begin{theorem}[FTPL: Online Linear Optimization]
\label{thm:ftplolo}
Let $\|\mathcal{X}\|_2 = \sup_{x\in\mathcal{X}} \|x\|_2$ and $\|\mathcal{Y}\|_2 = \sup_{l_t\in\mathcal{Y}} \|l_t\|_2$. Choosing $\sigma = \max\{\|\mathcal{Y}\|_2\sqrt{\frac{T}{\sqrt{N}\log T}}, \frac{\sqrt{N}}{\epsilon}\log T \log \frac{\log T}{\delta}\}$ and $\mathcal{D}=\mathcal{N}(0,\sigma^2\mathbb{I}_N)$, we have that $Regret_{\mathcal{A}(\mathcal{D},T)}(T)$ is 
\begin{equation*}
O\Bigg(N^{\frac{1}{4}}\|\mathcal{X}\|_2\|\mathcal{Y}\|_2\sqrt{T} + \frac{N\|\mathcal{X}\|_2}{\epsilon} \log^{1.5} T \log \frac{\log T}{\delta}\Bigg)
\end{equation*}
Moreover the algorithm is $\epsilon$-differentially private.
\end{theorem}

The proof of the theorem is deferred to the appendix.

\section{Differentially Private Multi-Armed Bandits}
In this section, we state our main results regarding bandit linear optimization, the algorithms that achieve it and prove the associated regret bounds. The following is our main theorem concerning {\em non-stochastic multi-armed bandits}.

\begin{theorem}[Differentially Private Multi-Armed Bandits]
\label{thm:mab}
Fix loss vectors $(l_1 \ldots l_T)$ such that $\|l_t\|_{\infty} \leq 1$. When Algorithm \ref{alg:DPBanditGeneral} is run with parameters $\mathcal{D} = Lap^N(\lambda)$ where $\lambda = \frac{1}{\epsilon}$ and algorithm $\mathcal{A} = \text{Algorithm \ref{alg:modifiedexp3}}$ with the following parameters: $\eta = \sqrt{\frac{\log N}{2 N T (1 + 2\lambda^2 \log NT)}}$, $\gamma = \eta N \sqrt{1 + 2 \lambda^2 \log NT}$ and the exploration distribution $\mu(i) = \frac{1}{N}$. The regret of the Algorithm \ref{alg:DPBanditGeneral} is
		\[ O\left( \frac{\sqrt{NT\log T\log N}}{\epsilon}\right) \]
Moreover, Algorithm \ref{alg:DPBanditGeneral} is $\epsilon$-differentially private
\end{theorem}

\textbf{Bandit Feedback: Reduction to the Non-private Setting}

We begin by describing an algorithmic reduction that takes as input a non-private bandit algorithm and translates it into an $\epsilon$-differentially private bandit algorithm. The reduction works in a straight-forward manner by adding the requisite magnitude of Laplace noise to ensure differential privacy. For the rest of this section, for ease of exposition we will assume that both $T$ and $N$ are sufficiently large. 

\begin{algorithm}[h!]
\caption{$\mathcal{A}'(\mathcal{A}, \mathcal{D})$ -- Reduction to the Non-private Setting for Bandit Feedback}
\label{alg:DPBanditGeneral}
\begin{algorithmic}[1]
		\REQUIRE Online Algorithm $\mathcal{A}$, Noise Distribution $\mathcal{D}$.
		\FOR{$t=0$ \textbf{to} $T$}
			\STATE Receive $\tilde{x_t}\in\mathcal{X}$ from $\mathcal{A}$ and output $\tilde{x}_t$.
			\STATE Receive a loss value $\langle l_t, \tilde{x_t}\rangle$ from the adversary.
			\STATE Sample $Z_t \sim \mathcal{D}$.
			\STATE Forward $\langle l_t\, \tilde{x_t}\rangle + \langle Z_t, \tilde{x_t}\rangle$ as input to $\mathcal{A}$.
		\ENDFOR
\end{algorithmic}
\end{algorithm}

\begin{algorithm}[h!]
	\caption{EXP2 with exploration $\mu$}
\label{alg:modifiedexp3}
\begin{algorithmic}[1]
		\REQUIRE learning rate $\eta$; mixing coefficient $\gamma$; distribution $\mu$
		\STATE $q_1 = \left( \frac{1}{N} \ldots \frac{1}{N}\right) \in \reals^{N} $. 
		\FOR{t = 1,2 \ldots T}
			\STATE Let $p_t = (1 - \gamma)q_t + \gamma \mu$ and play $i_t \sim p_t$
			\STATE Estimate loss vector $l_t$ by $\tilde{l_t} = P_t^{+} e_{i_t}e_{i_t}^Tl_t$, with $P_t = \av[i \sim p_t]{e_ie_i^T}$
			\STATE Update the exponential weights,
			\[ q_{t+1}(i) = \frac{e^{-\eta \langle e_i, \tilde{l}_t\rangle}q_t(i)}{\sum_{i'}e^{-\eta \langle e_{i'}, \tilde{l}_t\rangle}q_t(i')}\]
		\ENDFOR
\end{algorithmic}
\end{algorithm}

The following Lemma characterizes the conditions under which Algorithm \ref{alg:DPBanditGeneral} is $\epsilon$ differentially private

\begin{lemma}[Privacy Guarantees]
\label{lemma:privacybandit}
Assume that each loss vector $l_t$ is in the set $\mathcal{Y}\subseteq \mathbb{R}^N$, such that $\max_{t, l\in \mathcal{Y}} |\frac{\langle l, \tilde{x}_t\rangle}{\|\tilde{x}_t\|_\infty}|\leq B$. For $\mathcal{D}=Lap^N(\lambda)$ where $\lambda = \frac{B}{\epsilon}$, the sequence of outputs $(\tilde{x_t}:t\in [T])$ produced by the Algorithm $\mathcal{A}'(\mathcal{A},\mathcal{D})$ is $\epsilon$-differentially private.
\end{lemma}

 The following lemma charaterizes the regret of Algorithm \ref{alg:DPBanditGeneral}. In particular we show that the regret of Algorithm \ref{alg:DPBanditGeneral} is, in expectation, same as that of the regret of the input algorithm $\mathcal{A}$ on a perturbed version of loss vectors.

\begin{lemma} [Noisy Online Optimization]
\label{lemma:noisyoco}
Consider a loss sequence $(l_1 \ldots l_T)$ and a convex set $\mathcal{X}$. Define a perturbed version of the sequence as random vectors $(\tilde{l}_t: t\in [T])$ as $\tilde{l}_t = l_t + Z_t$ where $Z_t$ is a random vector such that $\{Z_1, \ldots Z_t\}$ are independent and $\mathbb{E}[Z_t]=0$ for all $t\in [T]$.

Let $\mathcal{A}$ be a full information (or bandit) online algorithm which outputs a sequence $(\tilde{x}_t\in \mathcal{X}:t\in [T])$ and takes as input $\tilde{l}_t$ (respectively $\langle \tilde{l}_t,\tilde{x}_t\rangle$) at time $t$. Let $x^* \in K$ be a fixed point in the convex set. Then we have that 
\[\mathbb{E}_{\{Z_t\}}\left[\mathbb{E}_{\mathcal{A}}\left[\sum_{t=1}^{T} \left( \langle l_t,\tilde{x}_t\rangle - \langle l_t,x^*\rangle\right)\right]\right] \]
\[ = \mathbb{E}_{\{Z_t\}}\left[\mathbb{E}_{\mathcal{A}}\left[\sum_{t=1}^{T} \left(\langle\tilde{l}_t.\tilde{x}_t \rangle- \langle\tilde{l}_t,x^*\rangle\right)\right] \right] \]
\end{lemma}

We provide the proof of Lemma \ref{lemma:privacybandit} and defer the proof of Lemma \ref{lemma:noisyoco} to the Appendix Section \ref{sec:OCOTheoremProof}.

\begin{proof}[Proof of Lemma \ref{lemma:privacybandit}]
Consider a pair of sequence of loss vectors that differ at exactly one time step -- say $L=(l_1,\dots l_{t_0}\dots,l_T)$ and $L'=(l_1,\dots, l'_{t_0}, \dots l_T)$. Since the prediction of produced by the algorithm at time step any time $t$ can only depend on the loss vectors in the past $(l_1,\dots l_{t-1})$, it is clear that the distribution of the output of the algorithm for the first $t_0$ rounds $(\tilde{x}_1,\dots \tilde{x}_{t_0})$ is unaltered. We claim that $\forall \mathcal{I}\subseteq \mathbb{R}$, it holds that
 \[\mathbb{P}(\langle l_{t_0}+Z_{t_0}, \tilde{x}_{t_0}\rangle \in \mathcal{I}) \leq e^\epsilon \mathbb{P}(\langle l'_{t_0}+Z_{t_0},\tilde{x}_{t_0}\rangle\in \mathcal{I}) \]
Before we justify the claim, let us see how this implies that desired statement. To see this, note that conditioned on the value fed to the inner algorithm $\mathcal{A}$ at time $t_{0}$, the distribution of all outputs produced by the algorithm are completely determined since the feedback to the algorithm at other time steps (discounting $t_0$) stays the same (in distribution). By the above discussion, it is sufficient  to demonstrate $\epsilon$-differential privacy for each input fed (as feedback) to the algorithm $\mathcal{A}$.

For the sake of analysis, define $\l^{Fict}_t$ as follows. If $\tilde{x}_t=0$, define $l^{Fict}_t = 0\in \mathbb{R}^N$. Else, define $l^{Fict}_t \in \mathbb{R}^N$ to be such that $(l^{Fict}_t)_i = \frac{\langle l_t, \tilde{x_t}\rangle}{\tilde{x}_i}$ if and only if $i=argmax_{i\in [d]} |\tilde{x}_i|$ and $0$ otherwise, where $argmax$ breaks ties arbitrarily. Define $\tilde{l}^{Fict}_t = l^{Fict}_t + Z_t$. Now note that $\langle \tilde{l}^{Fict}_t, \tilde{x}_t\rangle = \langle l_t\, \tilde{x_t}\rangle + \langle Z_t, \tilde{x_t}\rangle$.

It suffices to establish that each $\tilde{l}^{Fict}_t$ is $\epsilon$-differentially private. To argue for this, note that Laplace mechanism \cite{dpbook} ensures the same, since the $l_1$ norm of $\tilde{l}^{Fict}_t$ is bounded by $B$.
\end{proof}

\subsection{Proof of Theorem \ref{thm:mab}}

\textbf{Privacy:} Note that since $\max_{t, l\in \mathcal{Y}} |\frac{\langle l, \tilde{x}_t\rangle}{\|\tilde{x}_t\|_\infty}|\leq \|\mathcal{Y}\|_\infty \leq 1$ as $\tilde{x}_t\in \{e_i:i\in [N]\}$. Therefore by Lemma \ref{lemma:privacybandit}, setting $\lambda = \frac{1}{\epsilon}$ is sufficient to ensure $\epsilon$-differential privacy. 

\textbf{Regret Analysis:} For the purpose of analysis we define the following pseudo loss vectors.
	\[ \tilde{l}_t = l_t + Z_t\]
	where by definition $Z_t \sim Lap^N(\lambda)$. The following follows from Fact \ref{fact:infGaussianTail} proved in the appendix.
	\[ \mathbb{P}(\|Z_t\|_{\infty}^2 \geq 10 \lambda^2 \log^2 NT ) \leq \frac{1}{T^2}\]
	Taking a union bound, we have 
	\begin{equation}
	\label{eqn:infgaussiannormbound23}
	\mathbb{P}(\exists t\;\;\|Z_t\|_{\infty}^2 \geq 10 \lambda^2 \log^2 NT ) \leq \frac{1}{T}
	\end{equation}

	To bound the norm of the loss we define the event $F \defeq \{\exists t: \|Z_t\|_{\infty}^2 \geq 10 \lambda^2 \log^2 NT\}$. We have from \eqref{eqn:infgaussiannormbound23} that $\mathbb{P}(F) \leq \frac{1}{T}$. We now have that 
	\[ \mathbb{E}[Regret] \leq \mathbb{E}[Regret | \bar{F}] + \mathbb{P}(F)\mathbb{E}[Regret | F]\]
	Since the regret is always bounded by $T$ we get that the second term above is at most 1. Therefore we will concern ourselves with bounding the first term above. Note that $Z_t$ remains independent and symmetric even when conditioned on the event $\bar{F}$. Moreover the following statements also hold.
	\begin{equation}
	\label{eqn:infunbiasedconditioning}
		\forall t \;\; \mathbb{E}[Z_t | \bar{F}] = 0
	\end{equation}
	\begin{equation}
	\label{eqn:infvarianceconditioning}
		 \forall t \;\; \mathbb{E}[\|Z_t\|_{\infty}^2 | \bar{F}] \leq 10\lambda^2 \log^2 NT
	\end{equation}

	Equation \eqref{eqn:infunbiasedconditioning} follows by noting that $Z_t$ remains symmetric around the origin even after conditioning. It can now be seen that Lemma \ref{lemma:noisyoco} still applies even when the noise is sampled from $Lap^N(\lambda)$ conditioned under the event $\bar{F}$ (due to Equation \ref{eqn:infunbiasedconditioning}). Therefore we have that 
	\begin{equation}
	\label{eqn:infunbiasedregret}
	\mathbb{E}[Regret | \bar{F}] = \mathbb{E}_{\{Z_t\}}\left[\mathbb{E}_{\mathcal{A}}\left[\sum_{t=1}^{T} \left(\langle \tilde{l}_t ,\tilde{x}_t\rangle - \langle \tilde{l}_t , x^*\rangle\right)\right] \biggr| \bar{F}\right] \end{equation}
	
	To bound the above quantity we make use of the following lemma which is a specialization of Theorem 1 in \cite{bubecklinearbandit} to the case of multi-armed bandits.
\begin{lemma}[Regret Guarantee for Algorithm \ref{alg:modifiedexp3}]
\label{lemma:modifiedEXP3}
	 If $\eta$ is such that $\eta |\langle e_i,\tilde{l}_t\rangle| \leq 1$, we have that the regret of Algorithm \ref{alg:modifiedexp3} is bounded by 
	 \[Regret  \leq 2 \gamma T + \frac{\log N}{\eta} + \eta \mathbb{E}\bigg[\sum_t \sum_i p_t(i)\langle e_i, \tilde{l}_t\rangle^2 \bigg] \]
	
\end{lemma}

	Now note that due to the conditioning $\|Z_t\|_{\infty}^2 \leq 10 \lambda^2 \log^2 NT$ and therefore we have that 
	\[ max_{t,x\in \Delta_N} |\langle Z_t , x\rangle| \leq 4 \lambda \log NT. \]
	It can be seen that the condition $\eta |\langle e_i ,\tilde{l}_t\rangle|\leq 1$ in Theorem \ref{lemma:modifiedEXP3} is satisfied for exploration $\mu(i) = \frac{1}{N}$ and under the condition $\bar{F}$ as long as 
	\[ \eta N (1 + 4 \lambda \log NT) \leq \gamma\] 
	which holds by the choice of these parameters. Finally
	\begin{align*}
	& \mathbb{E}[Regret | \bar{F}] \\
	&= \mathbb{E}_{\{Z_t\}}\left[\mathbb{E}_{\mathcal{A}}\left[\sum_{t=1}^{T} \left(\langle\tilde{l}_t , \tilde{x}_t\rangle - \langle\tilde{l}_t,  x^*\rangle\right)\right] \biggr| \bar{F} \right] \\
	& \leq \mathbb{E}_{\{Z_t\}} \left[ \frac{\log N}{\eta} + \eta \sum_{t = 1}^{T}N \|\tilde{l}_t\|_{\infty}^2 +  2T \gamma \;\; \biggr| \bar{F} \right] \\
	& \leq \mathbb{E}_{\{Z_t\}} \left[ \frac{\log N}{\eta} + 2\eta \sum_{t = 1}^{T}N(\|l_t\|_{\infty}^2 + \|Z_t\|_{\infty}^2) +  2T \gamma\;\; \biggr| \bar{F} \right]\\
	& \leq \frac{\log N}{\eta} + 2\eta TN (1 + \lambda^2 \log^2 NT) + 2T \gamma  \\
	&\leq O\left( \sqrt{T N \log N (1 + \lambda^2 \log^2 NT)} \right) \\
	&\leq O\left( \frac{\sqrt{NT\log T\log N}}{\epsilon}\right)
\end{align*}

\subsection{Differentially Private Bandit Linear Optimization}
In this section we prove a general result about bandit linear optimization over general convex sets, the proof of which is deferred to the appendix.
\begin{theorem}[Bandit Linear Optimization]
\label{thm:LinearBanditGeneral}
Let $\mathcal{X}\subseteq \mathbb{R}^N$ be a convex set. Fix loss vectors $(l_1, \ldots l_T)$ such that $\max_{t, x \in \mathcal{X}} | \langle l_t , x \rangle | \leq M$. We have that Algorithm \ref{alg:DPBanditGeneral} when run with parameters $\mathcal{D}=Lap^N(\lambda)$ (with $\lambda = \frac{\|\mathcal{Y}\|_1}{\epsilon}$) and algorithm $\mathcal{A} = SCRiBLe$\cite{AHR} with step parameter $\eta = \sqrt{\frac{\nu \log T}{2 N^2 T (M^2 + \lambda^2 N \|\mathcal{X}\|_2^2)}}$ we have the following guarantees that the regret of the algorithm is bounded by
		\[ O\left( \sqrt{T \log T}\sqrt{N^2 \nu \left(M^2 + \frac{N\|\mathcal{X}\|_2^2\|\mathcal{Y}\|_1^2}{\epsilon^2}\right)} \right)\]
where $\nu$ is the self-concordance parameter of the convex body $\mathcal{X}$.
Moreover the algorithm is $\epsilon$-differentially private.
\end{theorem}

\section{Conclusion}
In this work, we demonstrate that ensuring differential privacy leads to only a constant additive increase in the incurred regret for online linear optimization in the full feedback setting. We also show nearly optimal bounds (in terms of T) in the bandit feedback setting. Multiple avenues for future research arise, including extending our bandit results to other challenging partial-information models such as semi-bandit, combinatorial bandit and contextual bandits. Another important unresolved question is whether it is possible to achieve an additive separation in $\epsilon, T$ in the adversarial bandit setting.

\bibliography{references}
\bibliographystyle{icml2017}

\newpage
\begin{appendix}

\section{Proofs for FTPL (Theorem 3.5)}

\begin{theorem}[Privacy Guarantees with Gaussian Noise]
\label{thm:p1}
	Choose any $\sigma^2 \geq \frac{\|\mathcal{Y}\|_2^2}{\epsilon^2}\log^2 T \log^2 \frac{\log T}{\delta}$. When Algorithm \ref{alg2} $\mathcal{A}(\mathcal{D},T)$ is run with $\mathcal{D}=\mathcal{N}(0,\sigma^2\mathbb{I}_N)$, the following claims hold true:
	\begin{itemize}
	\item {\bf Privacy:} The sequence $(\tilde{L}_t: t\in [T])$ is $(\epsilon,\delta)$-differentially private.
	\item {\bf Distribution:} It holds that $\forall t\in [T], \tilde{L}_t \sim \mathcal{N}(\sum_{s=1}^t l_s, \ceil{\log T}\sigma^2 \mathbb{I}_N)$.
\end{itemize}\end{theorem}
\begin{proof}
By Theorem 9 (\cite{jaink}), we have that the sequence $(\tilde{L'}_t: t\in [T])$ is $(\epsilon,\delta)$-differentially private. Now the sequence $(\tilde{L}_t: t\in [T])$ is $(\epsilon,\delta)$-differentially private because differential privacy is immune to post-processing \cite{dpbook}. 

Note that the $\texttt{PrivateSum}$ algorithm adds exactly $|\mathcal{S}|$ independent draws from the distribution $\mathcal{D}$ to $\sum_{s=1}^t l_s$, where $\mathcal{S}$ is the minimum set of already populated nodes in the tree that can compute the required prefix sum. Due to Line 6, it is made certain that every prefix sum released is a sum of the prefix sum and $\ceil{\log T}$ independent draws from $\mathcal{D}$.
\end{proof}

\begin{proof}[Proof of Theorem \ref{thm:ftplolo}]
As a consequence of Corollary 4 (\cite{olo}) and Lemma 13 (\cite{olo}), it is true that
\begin{equation*}
Regret_{\mathcal{A}(\mathcal{D},T)}(T)\leq \sigma\sqrt{N\log T}\|\mathcal{X}\|_2 + \frac{T\|\mathcal{X}\|_2\|\mathcal{Y}\|_2^2}{2\sigma\sqrt{\log T}}
\end{equation*}
Substituting the proposed value of $\sigma$, we obtain the stated claim.
\end{proof}
\section{Noisy OCO Theorem Proof}
\label{sec:OCOTheoremProof}
We now give the proof of Lemma \ref{lemma:noisyoco}.

\begin{proof}
	The proof of the lemma is a straightforward calculation. First note that
	\begin{multline*}
		\mathbb{E}_{\{Z_t\}}\left[\mathbb{E}_{\mathcal{A}}\left[\sum_{t=1}^{T} \left(\langle l_t,\tilde{x}_t\rangle - \langle l_t,x^*\rangle \right)\right]\right] \\
		= \mathbb{E}_{\{Z_t\}}\left[\mathbb{E}_{\mathcal{A}}\left[\sum_{t=1}^{T} \left(\langle\tilde{l}_t,\tilde{x}_t \rangle- \langle\tilde{l}_t,x^*\rangle\right)\right] \right] \\ + \mathbb{E}_{\{Z_t\}}\left[\mathbb{E}_{\mathcal{A}}\left[\sum_{t=1}^{T} \langle l_t -\tilde{l}_t,\tilde{x}_t\rangle \right]\right] - \mathbb{E}_{\{Z_t\}}\left[\mathbb{E}_{\mathcal{A}}\left[\sum_{t=1}^{T} \langle l_t -\tilde{l}_t,x^*\rangle \right]\right]
	\end{multline*}
	Now we have that 
	\begin{align*}
		&\mathbb{E}_{\{Z_t\}}\left[\mathbb{E}_{\mathcal{A}}\left[\sum_{t=1}^{T} \langle l_t -\tilde{l}_t,\tilde{x}_t\rangle \right]\right] \\
		&- \mathbb{E}_{\{Z_t\}}\left[\mathbb{E}_{\mathcal{A}}\left[\sum_{t=1}^{T} \langle l_t -\tilde{l}_t,x^*\rangle \right]\right] \\
		= &  \mathbb{E}_{\mathcal{A}} \left[ \sum_{t=1}^{T} \mathbb{E}_{\{Z_t\}}\left[\langle l_T - \tilde{l}_t,\tilde{x}_t - x^*\rangle\right] \right] \\
		= &\mathbb{E}_{\mathcal{A}} \left[ \sum_{t=1}^{T} \langle \mathbb{E}\left[Z_t\right],\mathbb{E}[\tilde{x}_t - x^*]\rangle\right] \\
		= & 0
	\end{align*}
	The first inequality follows because the randomness of the algorithm does not depend on $Z_t$. The second equality follows because $x_t$, being a function of the loss vectors of the previous round, is independent of the noise $Z_t$ and the third equality follows from the fact that $\mathbb{E}[Z_t] = 0$.
\end{proof}
\section{Facts about Norms of Laplace Vectors}

The following simple fact follows for Laplacian distributions. The general versions follow immediately. 

\begin{fact}
\label{fact:infGaussianTail}
	Let $Z \sim Lap^N(\lambda)$, then we have that 

	\[ \mathbb{P}(\|Z\|_{\infty}^2 \geq 10 \lambda^2 \log^2 TN) \leq \frac{1}{T^2} \]

\end{fact}

\begin{proof}
	We will show that for a fixed $i\in [N]$ 
	\[ \mathbb{P}( Z(i)^2 \geq 10 \lambda^2 \log^2 TN ) \leq \frac{1}{NT^2} \]

	The proof then follows by a simple union bound. Note that it is enough to bound the probability that 
	$P(Y \geq \sqrt{10} \lambda \log TN)$ where Y is distributed as an exponential random variable with parameter $\lambda$. This is bounded by $\frac{1}{T^2N}$. A simple union bound finishes the proof now.

\end{proof}


\section{Proof for Bandit Linear Optimization}
\begin{proof}[Proof of Theorem \ref{thm:LinearBanditGeneral}]
We prove the privacy guarantees first, followed by the regret bound. 

	\textbf{Privacy:} Note that $\max_{t, l\in \mathcal{Y}} |\frac{\langle l, \tilde{x}_t\rangle}{\|\tilde{x}_t\|_\infty}| \leq \|\mathcal{Y}\|_1$ by Holder's inequality. Therefore, with $\lambda = \frac{\|\mathcal{Y}\|_1}{\epsilon}$, $\mathcal{D}=Lap^N(\lambda)$ ensures $\epsilon$-differential privacy due to Lemma \ref{lemma:privacybandit}. 
	
	\textbf{Regret:} For the purpose of analysis we define the following pseudo loss vectors $\tilde{l}_t = l_t + Z_t$, where by definition $Z \sim Lap^N(\lambda)$. Further note that the following (which is a direct consequence of Fact \ref{fact:infGaussianTail} proved in the appendix) holds for $Z_t \sim Lap^N(\lambda)$
	\[ \mathbb{P}(\|Z_t\|_2^2 \geq 10 \lambda^2N \log^2 NT ) \leq \frac{1}{T^2}\]
	Therefore taking a union bound we get that 
	\begin{equation}
	\label{eqn:gaussiannormbound}
	\mathbb{P}(\exists t:\|Z_t\|_2^2 \geq 10 \lambda^2N  \log^2 NT ) \leq \frac{1}{T}
	\end{equation}
	We condition on the following event $F = \{\exists t: \|Z_t\|_2^2 \geq 10 \lambda^2N \log^2 NT\}$. We have from \eqref{eqn:gaussiannormbound} that $\mathbb{P}(F) \leq \frac{1}{T}$. We now have that 
	\[ \mathbb{E}[Regret] \leq \mathbb{E}[Regret | \bar{F}] + \mathbb{P}(F)\mathbb{E}[Regret | F]\]
	Since the regret is always bounded by $T$ we get that the second term above is at most 1. Therefore we will concern ourselves with bounding the first term above. For the rest of the section we will assume the conditioning on the event $\bar{F}$. First note that since the noise vectors $Z_t$ were independent to begin with they still remain conditionally independent even when conditioned on the event $\bar{F}$. The following two statements can be seen to follow easily. 
	\begin{equation}
	\label{eqn:unbiasedconditioning2}
		\forall t \;\; \mathbb{E}[Z_t | F] = 0
	\end{equation}
	\begin{equation}
	\label{eqn:varianceconditioning2}
		 \forall t \;\;\; \mathbb{E}[\|Z_t\|_2^2 | \bar{F}] \leq  10\lambda^2 N  \log^2 NT
	\end{equation}
	It follows from Equation \ref{eqn:unbiasedconditioning2} that Lemma \ref{lemma:noisyoco} applies even when the noise is sampled from $Lap^N(\lambda)$ conditioned on the event $\bar{F}$. Therefore we have that 
	\begin{equation}
	\label{eqn:unbiasedregret}
	\mathbb{E}[Regret | \bar{F}] = \mathbb{E}_{\{Z_t\}}\left[\mathbb{E}_{\mathcal{A}}\left[\sum_{t=1}^{T} \left(\langle \tilde{l}_t , \tilde{x}_t \rangle - \langle \tilde{l}_t, x^*\rangle\right) \right] \biggr| \bar{F} \right] \end{equation}
	Now note that since due to the conditioning $\|Z_t\|^2 \leq 10 \lambda^2N \log^2 NT$ and therefore we have that 
	\[ L \defeq max_{t,x\in \mathcal{X}} |\langle Z_t, x\rangle| \leq 4 \lambda \|\mathcal{X}\|_2 \sqrt{N \log^2 NT}. \]
	We now use the following theorem which is a simple restatement of Theorem 5.1 in \cite{AHR}. 


\begin{theorem}[Theorem 5.1 in \cite{AHR}]
\label{thm:ahr}
Fix a loss sequence $l_1,\ldots l_T$. Let $\mathcal{X}$ be a compact convex set and $\mathcal{R}$ be a $\nu$-self concordant barrier on $\mathcal{X}$. Assume $\max_{t,x\in\mathcal{X}}|\langle l_t,x\rangle| \leq M$. Setting $\eta \leq \frac{1}{4 N M}$ then the regret of the SCRiBLe algorithm is bounded by 
\[ Regret_{SCRiBLe} \leq 2\eta \sum N \langle l_t,x_t\rangle^2 + \nu \eta^{-1}\log T + 2M\]
where $(x_t\in\mathcal{X}:t\in [T])$ is the sequence of points played by the algorithm.
\end{theorem}

	It can now be verified that $\eta \leq \frac{1}{4N(M+L)}$. Therefore we can apply Theorem \ref{thm:ahr} obtain that 
	\begin{align*}
	&\mathbb{E}[Regret | \bar{F}] \\
	&= \mathbb{E}_{\{Z_t\}}\left[\mathbb{E}_{\mathcal{A}}\left[\sum_{t=1}^{T} \left(\langle \tilde{l}_t , \tilde{x}_t\rangle  - \langle \tilde{l}_t , x^*\rangle\right) \right] \biggr| \bar{F} \right] \\
	&\leq  \mathbb{E}_{\{Z_t\}} \left[ 2\eta \sum N^2|\langle \tilde{l}_t,\tilde{x}_t\rangle |^2 | \bar{F}\right] \\
	& \quad \quad + \mathbb{E}_{\{Z_t\}}\left[\nu \eta^{-1}\log T + 2(M+L) \; \biggr| \bar{F} \right] \\
	&\leq  \mathbb{E}_{\{Z_t\}} \left[ 2\eta \sum N^2( |\langle l_t,\tilde{x}_t\rangle|^2 + \|Z_t\|_2^2\|\tilde{x}_t\|_2^2) | \bar{F}\right]\\
	&\quad \quad + \mathbb{E}_{\{Z_t\}}\left[\nu \eta^{-1}\log T + 2(M+L) \;\; \biggr| \bar{F} \right]\\
	&\leq   2\eta T N^2( M^2 + \lambda^2 N \|\mathcal{X}\|_2^2) + \nu \eta^{-1}\log T \\ 
	&\quad \quad + 8 \lambda \|\mathcal{X}\|_2 \sqrt{N \log^2 NT}+2M \\
	&\leq  O\left( \sqrt{T \log T}\sqrt{N^2 \nu (M^2 + \lambda^2 N \|\mathcal{X}\|_2^2)} \right)
\end{align*}
\end{proof}
\end{appendix}

\end{document}